\documentclass[10pt]{article} 
\usepackage{amsthm}
\usepackage{graphicx}
\usepackage{booktabs} 
\newtheorem{definition}{Definition}
\newtheorem{proposition}{Proposition}
\newtheorem{assumption}{Assumption}
\newtheorem{theorem}{Theorem}
\newtheorem{corollary}{Corollary}
\newtheorem{lemma}{Lemma}

\usepackage[preprint]{tmlr}

\usepackage{amsmath,amsfonts,bm}

\def\eqref#1{equation~\ref{#1}}

\def\1{\bm{1}}

\DeclareMathAlphabet{\mathsfit}{\encodingdefault}{\sfdefault}{m}{sl}
\SetMathAlphabet{\mathsfit}{bold}{\encodingdefault}{\sfdefault}{bx}{n}

\usepackage[hidelinks]{hyperref}
\usepackage{url}

\title{Diffusion Operator Geometry of\\Feedforward Representations}

\author{\name Kanishka Reddy \email kani@uw.edu \\
      \addr Department of Applied Mathematics\\
      University of Washington}

\def\month{MM}  
\def\year{YYYY} 
\def\openreview{\url{https://openreview.net/forum?id=XXXX}} 

\newtheorem{remark}{Remark}
\begin{document}

\maketitle

\begin{abstract}
Feedforward neural networks transform data through learned representations
whose geometry shapes how classes separate and relate across successive
layers. We study that geometry through diffusion operators. Each feature-cloud
snapshot is assigned a Gaussian-kernel Markov operator, giving a smooth
description of one-step transport between classes from which spectral,
boundary, and local geometric information can be read. We define the empirical
class chain, state the condition under which it is an exact Markov quotient,
and derive both the corresponding population transition and a simpler overlap
chain based on expected class affinities. For balanced shared-covariance
Gaussian class-conditional snapshots these affinities have closed forms
controlled by a regularized Mahalanobis separation, which yields explicit
expressions for leakage and coarse spectral behaviour. We further show that
operator observables vary smoothly under feature perturbations, whereas hard
neighborhood graphs are controlled by neighbor-order margins. Experiments on
CIFAR-10 and CIFAR-100 ResNet-18 representations find that class transport
becomes increasingly persistent with depth while retaining structured
relations between classes, and that the diffusion class chain is more stable
than its $k$-nearest-neighbor counterpart under matched perturbations.
\end{abstract}

\section{Introduction}
\label{sec:intro}

Neural networks reshape data through their learned representations. If
$\Phi_\ell:\mathbb R^{d_0}\to\mathbb R^{d_\ell}$ is the map to layer
$\ell$, then the feature cloud
\[
Z^\ell=\{z_i^\ell\}_{i=1}^n,
\qquad
z_i^\ell=\Phi_\ell(x_i),
\]
evolves with depth and training. As this happens, the snapshots often become
more separable and lower-dimensional, with their organization increasingly
aligned with label structure
\citep{ansuini2019,alain2017,naitzat2020,kornblith2019,
papyan2020,zhu2021}. Their geometry therefore offers a natural way to study
representation learning.

Recent graph-geometric work approaches this evolution by constructing
neighborhood graphs on feature clouds and measuring discrete curvature or
related Ricci-flow-inspired quantities
\citep{baptista2024,hehl2025}. These studies connect geometric change with
class separation and community formation, although their primitive object is
a hard adjacency, which means that a small perturbation may rewire a
$k$-nearest-neighbor graph and cause the resulting measurements to depend on
both the graph construction and the chosen curvature discretization
\citep{ollivier2009,forman2003,garciatrillos2023}.

Here each representation snapshot is assigned a smooth diffusion-map Markov
operator \citep{coifman2006}. For bandwidth $\varepsilon>0$, set
\[
W_{ij}^\ell
=
\exp\!\left(
-\frac{\|z_i^\ell-z_j^\ell\|^2}{4\varepsilon}
\right),
\qquad
P^\ell=(D^\ell)^{-1}W^\ell,
\qquad
L^\ell=\frac{P^\ell-I}{\varepsilon}.
\]
The normalization serves as an operator-valued description of the learned
feature cloud rather than as an embedding algorithm. Its associated
carr\'e du champ is
\[
\Gamma^\ell(f,h)(i)
=
\frac{1}{2\varepsilon}
\sum_j P_{ij}^\ell
(f_j-f_i)(h_j-h_i).
\]
The operator $P^\ell$ gives class transport and coarse spectral structure,
while $\Gamma^\ell$ yields label-boundary and local quantities,
including soft diffusion radii and metric surrogates, so several geometric
summaries are read from the same Markov operator rather than attached as
separate scores to a chosen graph
\citep{bakry2014,jones2024,jones2026}.

Labels reveal the class structure carried by the operator. Averaging one-step
mass over source and destination classes gives a $K\times K$ transition
matrix whose diagonal entries measure class persistence, while the
off-diagonal entries retain relations that a scalar purity score would lose.
Throughout the paper, this matrix is treated as an aggregate one-step chain,
becoming an exact Markov quotient when the label partition satisfies the
standard lumpability condition.

The population formulation follows the same construction. Row-normalizing
the kernel at each source point and then averaging over a class gives the
exact population transition, whereas averaging the expected class affinities
first and normalizing afterward gives a simpler overlap chain. The two agree
when classwise kernel degrees are homogeneous, and
Proposition~\ref{prop:degree-heterogeneity} controls their discrepancy through
the amount of within-class degree variation.

A shared-covariance Gaussian mixture makes this distinction analytically
tractable. At a fixed layer and training time, suppose
\[
z\mid y=a \sim \mathcal N(\mu_a,\Sigma),
\qquad a\in[K],
\]
with balanced classes. The expected affinities depend on the pairwise
quantities
\[
c_\varepsilon^{(a,b)}
=
\tfrac14(\mu_a-\mu_b)^\top
(\varepsilon I+\Sigma)^{-1}
(\mu_a-\mu_b),
\]
which also determine the induced overlap chain. Under homogeneous class
geometry, leakage and the coarse spectrum become explicit monotone functions
of these regularized separations, although neither the sample construction
nor the general population definitions requires a Gaussian assumption.

The operator structure also gives perturbation control. For fixed
$\varepsilon>0$ and bounded feature clouds, the map
$Z\mapsto P_\varepsilon(Z)$ is locally Lipschitz, so aggregate class transport
and local geometric summaries change continuously with the features. A hard
$k$-nearest-neighbor graph behaves differently because its adjacency remains
fixed only while the neighbor-order margins stay open and may jump when one
of those margins closes, leaving smooth reweighting and combinatorial
rewiring as the two distinct mechanisms
\citep{ting2010,calder2022lipschitz}.

The framework is evaluated on three independently trained ResNet-18 models
for each of CIFAR-10 and CIFAR-100
\citep{he2016,krizhevsky2009}. Class persistence rises sharply with depth,
while the remaining off-diagonal transport retains recognizable semantic
structure, and on CIFAR-100 the residual mass also becomes increasingly
concentrated within the known semantic superclasses. Matched perturbation
experiments place the diffusion and directed $k$-nearest-neighbor class
chains on the same total-variation scale, with further analyses examining
degree heterogeneity, lumpability, density normalization, and variation from
the evaluation subset.

\section{Related work}
\label{sec:related}

\paragraph{Diffusion maps and operator geometry.}
Diffusion maps associate a finite sample with a Markov operator whose
spectrum and powers encode the geometry of the data
\citep{coifman2006,coifman2005pnas,singer2006,vonluxburg2007}. Their
continuum limits and dependence on the sampling distribution have been
studied through graph-Laplacian consistency results such as
\citet{hein2005}, while related operators have been used to estimate transfer
dynamics and construct nonlinear visualizations
\citep{klus2018transfer,moon2019}. We use the same diffusion-map
normalization, although our object of interest is the Markov operator itself
and the way its one-step mass moves between classes in a learned feature
space.

The operator also supports a local calculus through the carr\'e du champ,
which gives boundary energies, diffusion scales, and metric surrogates from
the same kernel construction
\citep{bakry2014,jones2024}. Supervised diffusion maps instead modify the
diffusion process using labels or a label kernel \citep{mendelman2025}. In
our construction the kernel depends only on the feature geometry, and labels
enter afterward when the resulting transport is averaged over source and
destination classes, so the class chain records structure already present in
the representation.

\paragraph{Distributional kernel summaries and Markov aggregation.}
For class distributions $\nu_a$ and $\nu_b$, the expected affinity
\[
A_{ab}
=
\mathbb E_{X\sim\nu_a,\,X'\sim\nu_b} k(X,X')
\]
is the inner product of their kernel mean embeddings
\citep{muandet2017,gretton2012}. The symmetric combination
$A_{aa}+A_{bb}-2A_{ab}$ gives the squared maximum mean discrepancy, while
the overlap chain keeps the individual cross-affinities and normalizes them
into directed transition probabilities. This viewpoint places the Gaussian
separation formulas inside a broader distributional kernel geometry and
explains why the population construction remains meaningful beyond the
Gaussian model.

Passing from a sample-level Markov chain to a class-level chain raises a
separate aggregation question. A partition gives an exact Markov quotient
only under the standard lumpability condition \citep{kemeny1976}, and the
statistical literature has developed tests and diagnostics for candidate
lumpings \citep{jernigan2003}. We state the corresponding condition for the
label partition, measure its failure directly, and use the aggregate class
matrix as a one-step transport summary when exact lumpability does not hold.

\paragraph{Geometry of learned representations.}
Neural feature geometry has been studied through intrinsic dimension,
separability, topology, representational similarity, and neural collapse
\citep{ansuini2019,alain2017,naitzat2020,kornblith2019,
papyan2020,zhu2021}. Other work asks how strongly the representation itself
moves during training, ranging from tangent-kernel descriptions of relatively
lazy regimes \citep{jacot2018,lee2018} to analyses that retain substantial
feature learning \citep{yanghu2021}. Our framework begins from the realized
feature snapshots and can therefore be applied without committing to a
particular training-limit description.

Several recent methods also summarize learned representations through
diffusion or global geometric statistics. Diffusion spectral entropy tracks
representation dynamics through a scalar spectral summary
\citep{dse2023}, while \citet{networkmanifold2024} embed entire trained
networks into a common comparison space. We retain the feature cloud at a
single depth and training time together with the full $K\times K$ matrix of
one-step class transport, preserving directional relations between classes
that a single global statistic would discard.

Recent work by \citet{khandait2026layers} uses diffusion-geometric Markov
operators to construct multiscale and multilayer measures of representational
similarity. We take a complementary perspective, aggregating an unsupervised
feature-space operator by labels, then distinguishing exact row-normalized population
transport from affinity-level overlap, which helps us to study class persistence,
lumpability, and perturbation response.

\paragraph{Graph curvature and Ricci flow on neural features.}
Ollivier curvature is built from optimal transport, while Forman curvature
arises from a combinatorial analogue of the Bochner identity
\citep{villani2008,ollivier2009,forman2003}. Graph-curvature and
Ricci-flow procedures have also been used to expose community structure by
evolving distances or edge weights
\citep{linluyau2011,ni2019}, and more recent work applies this perspective
to neural feature clouds, where layerwise curvature is interpreted through
class separation, community formation, and Ricci-flow-like evolution
\citep{baptista2024,hehl2025}.

The closest point of contact is \citet{hehl2025}, who study the layerwise
geometry of neural feature clouds through curvature on neighborhood graphs.
We ask the same broad question about how representation geometry changes
with depth, but begin from a diffusion operator rather than a hard
neighborhood graph, which allows us to describe one-step class transport
directly, derive an exact population transition and a tractable
affinity-based approximation, and control the response of the resulting
observables to perturbations of the features. The same change of primitive
also separates smooth reweighting of pairwise affinities from the discrete
rewiring of a neighborhood graph.

\section{Operator representation snapshots}
\label{sec:operator-snapshots}

\begin{definition}[Diffusion operator snapshot]
\label{def:operator-snapshot}
Let $Z=\{z_i\}_{i=1}^n\subset\mathbb R^d$. For $\varepsilon>0$, define
\[
W_{ij}=\exp\!\left(-\frac{\|z_i-z_j\|^2}{4\varepsilon}\right),
\qquad
D_i=\sum_jW_{ij},
\qquad
P_{ij}=\frac{W_{ij}}{D_i}.
\]
We call $P=P_\varepsilon(Z)$ the diffusion Markov operator of the snapshot,
with generator $L=(P-I)/\varepsilon$.
\end{definition}

The main text uses the raw-kernel random walk, which corresponds to
$\alpha=0$ in the Coifman--Lafon family. Replacing the kernel by
\[
W_{ij}^{(\alpha)}=\frac{W_{ij}}{(q_iq_j)^\alpha},
\qquad q_i=\sum_jW_{ij},
\]
gives the corresponding density-adjusted operator. Its affinity-level class
chain is recorded in Appendix~\ref{app:normalization}.

Suppose now that each point has a label $y_i\in[K]$. For any row-stochastic
matrix $Q$, let
\[
B_i^Q(b)=\sum_{j:y_j=b}Q_{ij}
\]
denote the one-step mass sent from node $i$ to class $b$. Averaging these
vectors within each source class gives
\[
T^{\mathrm{agg}}_{ab}(Q)
=
\frac{1}{n_a}\sum_{i:y_i=a}B_i^Q(b),
\qquad n_a=\#\{i:y_i=a\}.
\]
Thus every row-stochastic $Q$ induces a class-level one-step summary. It is an
exact Markov quotient only when nodes in the same class send identical mass
to every destination class,
\[
B_i^Q(\cdot)=B_{i'}^Q(\cdot)
\quad\text{whenever }y_i=y_{i'}.
\]
A convenient measure of departure from this lumpability condition is
\[
R_{\mathrm{lump}}(Q)
=
\frac1n\sum_{i=1}^n
\left\|B_i^Q-T^{\mathrm{agg}}_{y_i,\cdot}(Q)\right\|_1.
\]
The residual vanishes under exact lumpability. Otherwise,
$T^{\mathrm{agg}}$ still describes the average one-step transport between
classes, although iterating it need not reproduce the class projection of
$Q^t$.

The Gaussian kernel has the fixed self-weight $W_{ii}=1$. To prevent this
finite-sample self-loop from artificially increasing class persistence, the
empirical class chain is computed after removing the diagonal and
renormalizing each row,
\[
\widetilde P_{ij}
=
\frac{P_{ij}\mathbf 1\{i\ne j\}}
{\sum_{r\ne i}P_{ir}},
\]
so that class transport is summarized by
$T^{\mathrm{agg}}(\widetilde P)$. Quantities describing the local operator,
as well as raw cross-label mass, continue to use $P$.

\paragraph{Distinct leakage summaries.}
At the sample level, raw leakage is
\[
\mathcal L_{\mathrm{raw}}(P)
=
\frac1n\sum_i\sum_{j:y_j\ne y_i}P_{ij}.
\]
For a $K\times K$ row-stochastic class matrix $T$, mean persistence and its
uniform-source complement are
\[
\mathcal P(T)=\frac1K\operatorname{tr}(T),
\qquad
\mathcal L_{\mathrm{unif}}(T)
=1-\frac1K\operatorname{tr}(T).
\]
For a chosen stationary distribution $\pi^T$, stationary leakage is
\[
\mathcal L_{\mathrm{stat}}(T)
=
\sum_a\pi_a^T(1-T_{aa}),
\qquad \pi^T T=\pi^T.
\]
These quantities coincide only under additional symmetry, so they are kept
separate throughout.

The discrete carr\'e du champ associated with $P$ is
\[
\Gamma(f,h)(i)
=
\frac{1}{2\varepsilon}\sum_jP_{ij}(f_j-f_i)(h_j-h_i).
\]
For the one-hot label map $g(i)=e_{y_i}$, summing its coordinate energies gives
\[
\mathcal E_{\mathrm{label}}
=
\frac1n\sum_i\sum_{a=1}^K\Gamma(g_a)(i)
=
\frac{\mathcal L_{\mathrm{raw}}(P)}{\varepsilon}.
\]
The scalar total is therefore a rescaled form of raw leakage, while the
classwise terms show where the cross-label mass is concentrated.

Local spread around node $i$ is summarized by the soft diffusion radius
\[
\rho_\varepsilon^2(i)=\sum_jP_{ij}\|z_j-z_i\|^2.
\]
We report the second-largest eigenvalue modulus as a coarse spectral summary. Ordering
the eigenvalues so that $\lambda_1=1$ and
$|\lambda_1|\ge|\lambda_2|\ge\cdots$, define
\[
\operatorname{SLEM}(T)=|\lambda_2(T)|,
\qquad
\operatorname{gap}(T)=1-\operatorname{SLEM}(T).
\]

\section{Population operator geometry}
\label{sec:population-geometry}

\subsection{Exact row-normalized population transport and the overlap chain}
\label{sec:distribution-free-overlap}

Let $\nu_a$ be the class-conditional representation law and $\pi_a$ its
prior. Define
\[
r_b(x)=\int k_\varepsilon(x,x')\,d\nu_b(x'),
\qquad
s(x)=\sum_{c=1}^K\pi_c r_c(x).
\]
Averaging and row normalization can be ordered in two ways. Normalizing around
each source point first gives the exact one-step population class transition
\[
T^{\mathrm{pop}}_{ab}
=
\mathbb E_{X\sim\nu_a}
\left[\frac{\pi_b r_b(X)}{s(X)}\right].
\]
The same order of operations appears in the sample construction.

The overlap construction instead begins with the class affinities
\[
A_{ab}
=
\mathbb E_{X\sim\nu_a,X'\sim\nu_b}k_\varepsilon(X,X')
=
\mathbb E_{X\sim\nu_a}r_b(X),
\qquad
\bar s_a=\sum_c\pi_cA_{ac},
\]
and normalizes only after source-class averaging,
\[
T^{\mathrm{ov}}_{ab}
=
\frac{\pi_bA_{ab}}{\bar s_a}.
\]
We call $T^{\mathrm{ov}}$ the overlap chain. No Gaussian assumption is needed
at this stage: once the pairwise affinities are known, the chain follows
directly.

\begin{proposition}[Degree-heterogeneity control]
\label{prop:degree-heterogeneity}
For every source class $a$ with $\bar s_a>0$,
\[
\left\|T^{\mathrm{pop}}_{a\cdot}-T^{\mathrm{ov}}_{a\cdot}\right\|_1
\le
\frac{\mathbb E_{X\sim\nu_a}|s(X)-\bar s_a|}{\bar s_a}
\le
\frac{\sqrt{\operatorname{Var}_{\nu_a}(s(X))}}{\bar s_a}.
\]
Thus the overlap chain is accurate when the population kernel degree is
nearly homogeneous within each source class.
\end{proposition}

\begin{proof}
Using $\sum_b\pi_br_b(X)=s(X)$,
\begin{align*}
\left\|T^{\mathrm{pop}}_{a\cdot}-T^{\mathrm{ov}}_{a\cdot}\right\|_1
&\le
\mathbb E\left[
\sum_b\pi_br_b(X)
\left|\frac1{s(X)}-\frac1{\bar s_a}\right|
\right]\\
&=
\frac{\mathbb E|s(X)-\bar s_a|}{\bar s_a}.
\end{align*}
The variance bound follows from Cauchy--Schwarz.
\end{proof}

The bound ties the approximation error directly to variation of the
population kernel degree within each source class. In the experiments we
report the corresponding sample coefficient of variation alongside the
overlap-chain results.

A symmetric affinity separation is
\[
\widetilde c_\varepsilon^{(a,b)}
=
-\log\frac{A_{ab}}{\sqrt{A_{aa}A_{bb}}}.
\]
The Gaussian model below makes this quantity explicit.

\subsection{Exact Gaussian affinities and induced overlap chain}
\label{sec:gaussian-specialization}

\begin{assumption}[Balanced Gaussian snapshot model]
\label{assump:gaussian-snapshot}
At a fixed layer and training time,
$z\mid y=a\sim\mathcal N(\mu_a,\Sigma)$ for $a\in[K]$, with
$\mathbb P(y=a)=1/K$ and shared covariance $\Sigma\succeq0$.
\end{assumption}

Define
\[
c_\varepsilon^{(a,b)}
=
\frac14(\mu_a-\mu_b)^\top
(\varepsilon I+\Sigma)^{-1}
(\mu_a-\mu_b),
\qquad
c_\varepsilon^{(a,a)}=0.
\]

\begin{theorem}[Exact Gaussian affinities and induced overlap chain]
\label{thm:gaussian-bridge}
Let $X_a\sim\mathcal N(\mu_a,\Sigma)$ and
$X_b\sim\mathcal N(\mu_b,\Sigma)$ be independent. Then
\[
A_{ab}
=
\mathbb E[k_\varepsilon(X_a,X_b)]
=
A_0e^{-c_\varepsilon^{(a,b)}},
\qquad
A_0=\det\!\left(I+\frac{\Sigma}{\varepsilon}\right)^{-1/2}.
\]
Consequently, the balanced overlap chain is
\[
T^{\mathrm{ov}}_{ab}
=:\bar P_{ab}
=
\frac{e^{-c_\varepsilon^{(a,b)}}}
{\sum_{r=1}^K e^{-c_\varepsilon^{(a,r)}}}.
\]
The complete overlap-chain transition structure is determined by the pairwise
separations $c_\varepsilon^{(a,b)}$, and
$\widetilde c_\varepsilon^{(a,b)}=c_\varepsilon^{(a,b)}$.
\end{theorem}

The theorem gives the affinity matrix and its induced overlap chain in closed
form. Proposition~\ref{prop:degree-heterogeneity} relates this chain to the
row-normalized population transition. Shared covariance makes the Gaussian
integral tractable, while the remaining approximation is governed by how much
$s(X)$ varies within a source class.

Let
\[
w_{ab}=e^{-c_\varepsilon^{(a,b)}},
\qquad
q_a=\sum_rw_{ar}.
\]
Since $w_{ab}=w_{ba}$, the overlap chain is reversible with stationary
distribution
\[
\bar\pi_a=\frac{q_a}{\sum_sq_s}.
\]

\begin{proposition}[Stationary leakage of the Gaussian overlap chain]
\label{prop:population-leakage}
Under Assumption~\ref{assump:gaussian-snapshot},
\[
\mathcal L_{\mathrm{stat}}(T^{\mathrm{ov}})
=
\frac{\sum_{a\ne b}e^{-c_\varepsilon^{(a,b)}}}
{K+\sum_{a\ne b}e^{-c_\varepsilon^{(a,b)}}}.
\]
\end{proposition}

\begin{corollary}[Homogeneous separation]
\label{cor:homogeneous}
If $c_\varepsilon^{(a,b)}=c$ for all $a\ne b$, then
\[
\mathcal L_{\mathrm{stat}}(T^{\mathrm{ov}})
=
\frac{(K-1)e^{-c}}{1+(K-1)e^{-c}},
\qquad
\lambda_1=1,
\qquad
\lambda_2=\cdots=\lambda_K
=
\frac{1-e^{-c}}{1+(K-1)e^{-c}}.
\]
The SLEM is $\lambda_2$ and
\[
\operatorname{gap}(T^{\mathrm{ov}})
=
\frac{Ke^{-c}}{1+(K-1)e^{-c}}.
\]
\end{corollary}

\begin{remark}[Coarse and local information]
\label{rem:coarse-local}
Under Assumption~\ref{assump:gaussian-snapshot}, the affinity matrix is
determined by the common scale factor $A_0$ together with the pairwise
separations $c_\varepsilon^{(a,b)}$. After row normalization, the factor
$A_0$ cancels, so $T^{\mathrm{ov}}$, its stationary leakage, and its coarse
spectrum depend only on the separation matrix
$\{c_\varepsilon^{(a,b)}\}_{a,b}$. Kernel-tilted overlap proxies for local $\Gamma$ quantities, and soft
diffusion radius derived in Appendix~\ref{app:local-geometry} retain
covariance directions and the orientations of the mean differences.
\end{remark}

\section{Stability of operator observables}
\label{sec:stability}

For a differentiable feature path
$t\mapsto Z_t=(z_1(t),\ldots,z_n(t))$, let
\[
v_t=\max_i\|\dot z_i(t)\|.
\]

\begin{theorem}[Local data-dependent operator stability]
\label{thm:local-operator-stability}
For row $i$ of $P_\varepsilon(Z_t)$,
\[
\left\|\frac{d}{dt}P_i(Z_t)\right\|_1
\le
\frac{2v_t}{\varepsilon}\rho_\varepsilon(i;Z_t).
\]
Consequently,
\[
\|P_i(Z_1)-P_i(Z_0)\|_1
\le
\frac{2}{\varepsilon}
\int_0^1v_t\rho_\varepsilon(i;Z_t)\,dt.
\]
\end{theorem}

\begin{proof}
Writing
\[
a_{ij}
=
\frac{d}{dt}\log W_{ij}
=
-\frac{
\langle z_i-z_j,\dot z_i-\dot z_j\rangle
}{2\varepsilon}
\]
gives
\[
\dot P_{ij}
=
P_{ij}
\left(
a_{ij}-\sum_rP_{ir}a_{ir}
\right).
\]
Hence
\[
\|\dot P_i\|_1
\le
2\sum_jP_{ij}|a_{ij}|.
\]
Since
\[
|a_{ij}|
\le
\frac{v_t}{\varepsilon}\|z_i-z_j\|,
\]
Cauchy--Schwarz yields the bound.
\end{proof}

\begin{theorem}[Uniform stability on bounded feature clouds]
\label{thm:operator-stability}
Fix $\varepsilon>0$ and $R>0$. If
$\max_i\|z_i\|\le R$ and $\max_i\|\widetilde z_i\|\le R$, then
\[
\|P_\varepsilon(Z)-P_\varepsilon(\widetilde Z)\|_{\infty\to\infty}
\le
\frac{4R}{\varepsilon}e^{R^2/\varepsilon}
\|Z-\widetilde Z\|_\infty.
\]
\end{theorem}

The proof is in Appendix~\ref{app:stability}. Linear aggregation then gives
the following consequence.

\begin{corollary}[Stability of snapshot summaries]
\label{cor:observable-stability}
Under the assumptions of Theorem~\ref{thm:operator-stability},
$T^{\mathrm{agg}}(P)$, $\mathcal L_{\mathrm{raw}}(P)$, label-boundary
energies, and squared soft diffusion radii are Lipschitz functions of $Z$ for
fixed labels. The same conclusion holds for
$T^{\mathrm{agg}}(\widetilde P)$ when $n>1$, since the bounded-cloud
assumptions give a uniform positive lower bound on every off-diagonal row
mass.
\end{corollary}

For hard neighborhoods, let $d_{i,(k)}$ be the $k$-th ordered distance and
define the margin
\[
\Delta_i(Z)=d_{i,(k+1)}(Z)-d_{i,(k)}(Z).
\]

\begin{proposition}[Finite-perturbation $k$-NN stability is margin controlled]
\label{prop:knn-margin}
If $\|Z-\widetilde Z\|_\infty\le\delta$ and $\Delta_i(Z)>4\delta$, then the
directed $k$-nearest-neighbor set of point $i$ is unchanged. Any row whose directed neighbor set changes between $Z$ and
$\widetilde Z$ must belong to
\[
\{i:\Delta_i(Z)\le4\delta\}.
\]
Consequently, the number of changed rows is at most
\[
\#\{i:\Delta_i(Z)\le4\delta\},
\]
and the directed edge-set symmetric difference is at most
\[
2k\,\#\{i:\Delta_i(Z)\le4\delta\}.
\]
\end{proposition}

\begin{proof}
Every pairwise distance changes by at most $2\delta$. Original top-$k$
distances are at most $d_{i,(k)}+2\delta$, while excluded distances are at
least $d_{i,(k+1)}-2\delta$. The ordering is preserved when the gap exceeds
$4\delta$.
\end{proof}

\begin{proposition}[Discontinuity at neighbor ties]
\label{prop:knn-discontinuity}
For $1\le k<n-1$, the map from a feature cloud to its $k$-NN adjacency is
discontinuous at any cloud where a $k$-th and $(k+1)$-st neighbor distance
tie. Any adjacency-dependent observable that distinguishes the two resulting
graphs inherits this discontinuity.
\end{proposition}

The two results describe different mechanisms. Operator rows move in total
variation, whereas hard neighborhoods change through edge replacements. The
experiments report both mechanisms and also compare class transport with the
same total-variation statistic.

\section{Empirical evaluation}
\label{sec:eval}

\subsection{Protocol}
\label{sec:cifar-protocol}

We train three CIFAR-adapted ResNet-18 models for each dataset, and
obtain mean test accuracies of $95.30\pm0.37\%$ on CIFAR-10 and
$78.64\pm0.13\%$ on CIFAR-100. We extract features from the stem and the four
residual stages, apply global average pooling, and construct snapshots containing 100 examples per class for CIFAR-10 and 20
examples per class for CIFAR-100.

For each layer and subset we define
$\varepsilon_{\mathrm{med}}$ as one quarter of that snapshot's median
pairwise squared distance and evaluate
\[
\varepsilon/\varepsilon_{\mathrm{med}}
\in\{0.25,0.5,1,2\}.
\]
We refer to $0.25\varepsilon_{\mathrm{med}}$ as the finest tested bandwidth.
Because the reference bandwidth is specific to each snapshot, the depthwise results
describe class organization relative to each snapshot's global pairwise
distance scale. Features are centered before the operator is
built, which is numerically convenient and leaves every pairwise Euclidean
distance unchanged. A directed 20-nearest-neighbor graph serves as the hard
baseline, and all multiscale measurements are computed on the same fixed
balanced subset, so we can compare changes across bandwidth, normalization,
and depth. For finest-bandwidth stem and layer-4 results, the analysis is repeated
over five independently drawn balanced subsets for each trained model.
Perturbation curves use relative noise scales
$\{0.002,0.005,0.01,0.02,0.05\}$ and average ten noise draws per model at
each scale. The diffusion operator uses
$\varepsilon=0.25\varepsilon_{\mathrm{med}}$ and $\alpha=0$, while the hard
baseline is a directed $k$-nearest-neighbor graph with $k=20$. Bandwidth
is computed from the unperturbed layer-4 snapshot and held fixed after
perturbation.

At relative scale $\tau$, the perturbation is
\[
\xi_i
\sim
\mathcal N\!\left(
0,\frac{\tau^2\ell^2}{d}I_d
\right),
\qquad
\ell
=
\left(
\operatorname{median}_{r<s}\|z_r-z_s\|^2
\right)^{1/2}.
\]
Thus
\[
\left(\mathbb E\|\xi_i\|^2\right)^{1/2}=\tau\ell.
\]

\subsection{Multiscale class transport across depth}
\label{sec:cifar-transport}

Across both datasets, the class chains become progressively sharper with
depth. On the fixed CIFAR-10 subset, mean persistence at
$0.25\varepsilon_{\mathrm{med}}$ rises from $0.145\pm0.001$ in the stem to
$0.712\pm0.010$ at layer 4, and although broader kernels retain more
long-range transport, the same reorganization of the representation remains
visible throughout the bandwidth range
(Figure~\ref{fig:cifar-main}a).

The layer-4 transition matrix shows that the off-diagonal mass is structured
rather than diffuse. By this point in the network, most transport remains
within class, while the largest residual pathways connect semantically
related categories, most clearly cats with dogs and airplanes with ships,
with additional concentration among the other animal classes
(Figure~\ref{fig:cifar-main}b). The final representation therefore combines
strong class persistence with a smaller but still interpretable pattern of
cross-class transport.

Let $s(a)$ denote the semantic superclass of fine class $a$. For a class
transition matrix $T$, we define within-superclass off-diagonal enrichment by
\[
\operatorname{Enrich}_{\mathrm{sc}}(T)
=
\frac{1}{K}
\sum_{a=1}^K
\frac{
\displaystyle
\sum_{\substack{b\neq a\\s(b)=s(a)}}T_{ab}
}{
\displaystyle
(1-T_{aa})(4/99)
}.
\]
Thus the diagonal is removed before the remaining mass is normalized, and a
value of one is the uniform off-diagonal destination baseline.

CIFAR-100 exhibits the same depthwise sharpening at a finer semantic
resolution, with finest-bandwidth persistence increasing from
$0.0228\pm0.0002$ in the stem to $0.1355\pm0.0016$ at layer 4. At the same
time, the remaining off-diagonal mass becomes increasingly concentrated
within the dataset's 20 superclasses, as the within-superclass fraction rises
from $0.0543\pm0.0002$ to $0.0836\pm0.0010$. Relative to the
random-destination baseline of $4/99$, this corresponds to an enrichment
increase from $1.34\pm0.00$ to $2.07\pm0.03$, showing that the later
representations preserve a meaningful semantic organization even after the
diagonal mass is removed.

The same trends persist when the evaluation subset is resampled. Across the
15 model--subset pairs, CIFAR-10 persistence rises from
$0.1483\pm0.0035$ in the stem to $0.7178\pm0.0075$ at layer 4, while
CIFAR-100 rises from $0.02295\pm0.00072$ to $0.1368\pm0.0018$. At layer 4,
the mean within-model subset standard deviation is $0.0058$ on CIFAR-10 and
$0.0017$ on CIFAR-100, compared with standard deviations of $0.0061$ and
$0.0008$ across the model means, which shows that evaluation sampling
contributes materially to the reported variation and is especially
important on CIFAR-100.

\subsection{Degree heterogeneity and aggregation residuals}
\label{sec:empirical-degree-lumpability}

The behavior of the kernel degree differs sharply between the two datasets
at the finest bandwidth. On CIFAR-10, the mean within-class degree
coefficient of variation at $0.25\varepsilon_{\mathrm{med}}$ falls from
$0.598$ in the stem to $0.101$ at layer 4, so the final representation is
much more homogeneous at the local kernel scale. CIFAR-100 changes
far less over the same depth range, moving only from $0.631$ to $0.618$,
although its layer-4 value falls to $0.212$ once the bandwidth is increased
to $\varepsilon_{\mathrm{med}}$. The degree-heterogeneity bound is hence more favorable for
the final CIFAR-10 representation, compared to the finest-scale CIFAR-100 chain
which still carries substantial within-class degree variation.

The lumpability residual improves with both depth and bandwidth. At
$0.25\varepsilon_{\mathrm{med}}$, $R_{\mathrm{lump}}$ falls from $0.365$
to $0.224$ on CIFAR-10 and from $0.478$ to $0.256$ on CIFAR-100, while at
$\varepsilon_{\mathrm{med}}$ the corresponding layer-4 values decrease
further to $0.063$ and $0.038$. The label partition is consequently closer to one-step lumpability,
especially at the broader bandwidth, although the residual remains nonzero. Full results are
reported in Appendix~\ref{app:empirical-degree-lumpability}.

Density normalization is most influential in the same fine-scale regime.
Changing from $\alpha=0$ to $\alpha=1$ at
$0.25\varepsilon_{\mathrm{med}}$ shifts final-stage stationary leakage by
about $0.005$ on CIFAR-10, whereas on CIFAR-100 it changes the value from
$0.874$ to $0.770$. Once the bandwidth reaches
$\varepsilon_{\mathrm{med}}$, both differences fall to about $0.002$, which is
consistent with the reduction in degree heterogeneity at broader scales.

\begin{figure*}[t]
\centering
\includegraphics[width=\textwidth]{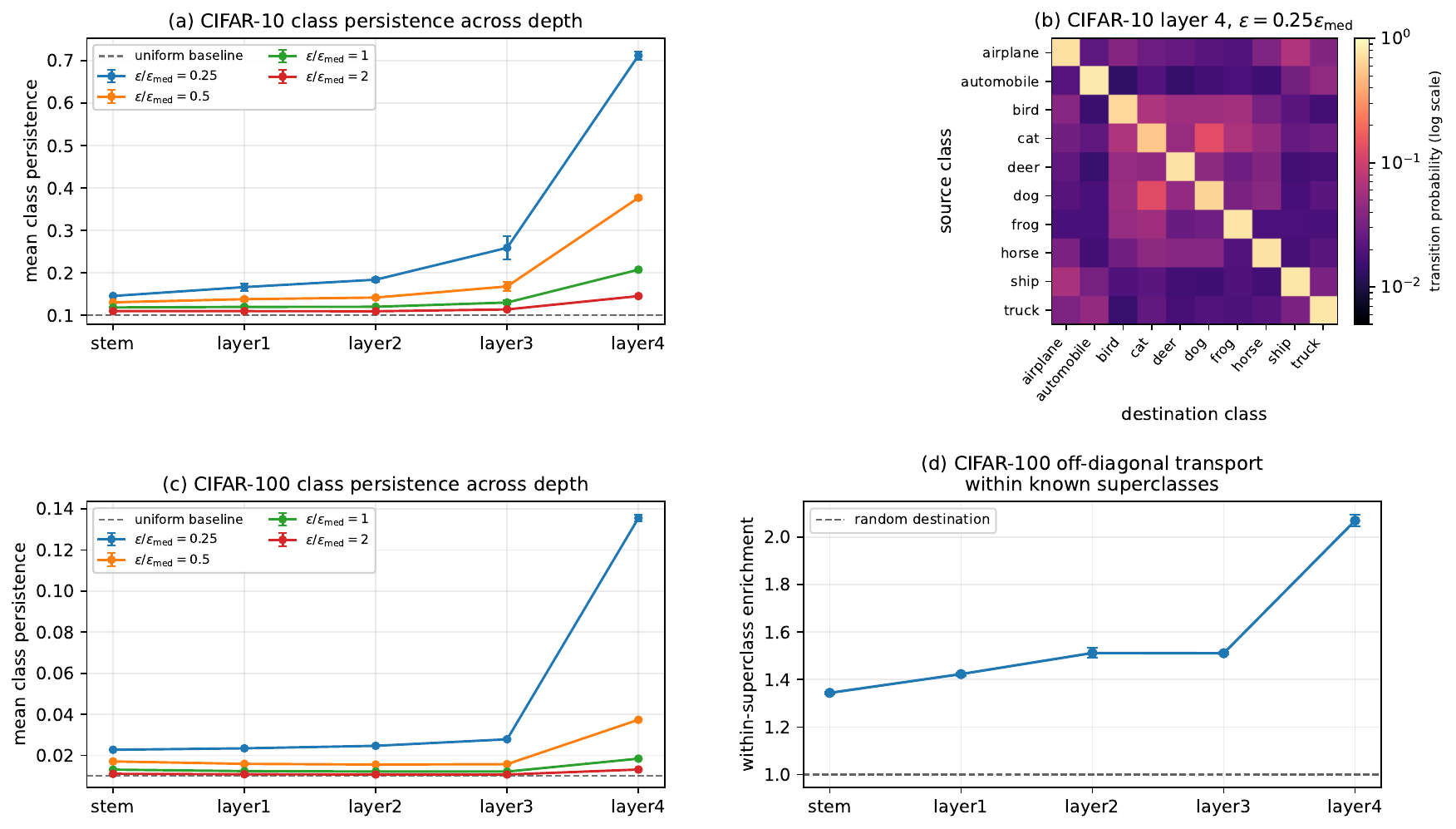}
\caption{\textbf{Multiscale class transport on residual networks.}
Error bars show one standard deviation across three independently trained
models evaluated on the same fixed balanced subset. (a) CIFAR-10 mean class
persistence across depth. (b) Mean CIFAR-10 layer-4 transition matrix at
$0.25\varepsilon_{\mathrm{med}}$ and $\alpha=0$, with class names shown on
both axes and transition probabilities displayed on a logarithmic colour
scale. (c) CIFAR-100 mean class persistence across depth. (d) Enrichment of
off-diagonal CIFAR-100 transport within the known semantic superclasses,
with one denoting the random-destination baseline.}
\label{fig:cifar-main}
\end{figure*}

\begin{table}[t]
\centering
\small
\setlength{\tabcolsep}{5pt}
\begin{tabular}{lcc}
\toprule
& CIFAR-10 & CIFAR-100\\
\midrule
Test accuracy & $95.30\pm0.37\%$ & $78.64\pm0.13\%$\\
Stem persistence$^{\dagger}$ & $0.1483\pm0.0035$ & $0.02295\pm0.00072$\\
Layer-4 persistence$^{\dagger}$ & $0.7178\pm0.0075$ & $0.1368\pm0.0018$\\
Layer-4 persistence / $1/K$$^{\dagger}$ & $7.18\times$ & $13.68\times$\\
Superclass enrichment, stem & -- & $1.344\pm0.004$\\
Superclass enrichment, layer 4 & -- & $2.069\pm0.025$\\
\bottomrule
\end{tabular}
\caption{Summary at $0.25\varepsilon_{\mathrm{med}}$ and $\alpha=0$.
$^{\dagger}$Persistence is computed over three trained models and five
balanced evaluation subsets per model, with the reported standard deviation
taken over the resulting 15 model--subset pairs. Test accuracy and
superclass enrichment report variation across the three trained models, and
superclass enrichment is computed on the fixed subset used in
Figure~\ref{fig:cifar-main}.}
\label{tab:cifar-summary}
\end{table}

\subsection{Perturbation response}
\label{sec:cifar-stability}

For the tested finest-bandwidth diffusion operator and directed $k=20$
baseline, the diffusion class chain changes less under every applied
perturbation when both are measured by the same total-variation statistic.
The perturbations and fixed-bandwidth convention are specified in
Section~\ref{sec:cifar-protocol}. Across the full noise range, the directed
$k$-NN chain moves between $4.2$ and $6.8$ times more on CIFAR-10 and between
$27.6$ and $40.4$ times more on CIFAR-100.

At the largest perturbation, mean total variation on CIFAR-10 is
$1.78\times10^{-4}$ for diffusion and $7.52\times10^{-4}$ for $k$-NN,
while the corresponding CIFAR-100 values are $4.01\times10^{-4}$ and
$1.11\times10^{-2}$. The complete perturbation curves appear in
Appendix~\ref{app:matched-stability}.

The mechanism-specific diagnostics separate the two forms of change.
Diffusion rows move continuously as the features are perturbed, whereas the
directed graph changes through discrete neighbor replacements.

\section{Discussion}
\label{sec:discussion}

This paper develops a diffusion-operator view of feedforward neural
representation geometry. Each feature snapshot is represented by a
Gaussian-kernel Markov operator whose one-step transport between classes
records how strongly class structure has formed and which relations remain
after separation. Across CIFAR-10 and CIFAR-100 ResNet-18 representations the
resulting class chains become increasingly persistent with depth, while their
off-diagonal mass continues to follow recognizable semantic structure. The
population analysis clarifies how this class-level picture arises from the
underlying operator, and the shared-covariance Gaussian model gives a solvable
case in which regularized Mahalanobis separation determines the class
affinities, leakage, and coarse spectrum.

The framework is built around finite feedforward snapshots, but the
class-level description also comes with ways of checking how faithfully it
represents the operator beneath it. Degree heterogeneity controls the gap
between the exact population transition and the affinity-based overlap chain,
while the lumpability residual measures how much variation is lost when
sample-level transition rows are averaged by label. In the CIFAR experiments
these diagnostics generally improve with depth and with a wider bandwidth,
although CIFAR-100 retains noticeable heterogeneity at the finest scales.
At fixed bandwidth, nearby feature clouds induce nearby diffusion operators
and class summaries. For the tested finest-bandwidth diffusion operator and
directed $k=20$ baseline, the diffusion class chain changes less under the
applied perturbations. Hard neighborhoods remain tied to neighbor-order
margins and may rewire abruptly when one of those margins closes. A natural
next step is to track these operators throughout training, rather than only
across depth in fully trained networks.

\bibliography{references}

@article{baptista2024,
  author  = {Baptista, Anthony and Barp, Alessandro and Chakraborti, Tapabrata
             and Harbron, Chris and MacArthur, Ben D. and Banerji, Christopher R. S.},
  title   = {Deep learning as {R}icci flow},
  journal = {Scientific Reports},
  volume  = {14},
  number  = {1},
  pages   = {23383},
  year    = {2024},
  doi     = {10.1038/s41598-024-74045-9},
}

@misc{khandait2026layers,
  title         = {From Layers to Networks: Comparing Neural Representations
                   via Diffusion Geometry},
  author        = {Khandait, Atharva and Gerken, Jan E.},
  year          = {2026},
  eprint        = {2605.15901},
  archivePrefix = {arXiv},
  primaryClass  = {cs.LG}
}

@inproceedings{hehl2025,
  author    = {Hehl, Moritz and von Renesse, Max-K. and Weber, Melanie},
  title     = {Neural Feature Geometry Evolves as Discrete {R}icci Flow},
  booktitle = {International Conference on Machine Learning (ICML)},
  year      = {2026},
  note      = {Spotlight},
  url       = {https://openreview.net/forum?id=YPH5yCKzYr},
}

@article{ni2019,
  author  = {Ni, Chien-Chun and Lin, Yu-Yao and Luo, Feng and Gao, Jie},
  title   = {Community Detection on Networks with {R}icci Flow},
  journal = {Scientific Reports},
  volume  = {9},
  number  = {1},
  pages   = {9984},
  year    = {2019},
  doi     = {10.1038/s41598-019-46380-9},
}

@article{gretton2012,
  author  = {Gretton, Arthur and Borgwardt, Karsten M. and Rasch, Malte J.
             and Sch{\"o}lkopf, Bernhard and Smola, Alexander},
  title   = {A Kernel Two-Sample Test},
  journal = {Journal of Machine Learning Research},
  volume  = {13},
  pages   = {723--773},
  year    = {2012},
}

@article{moon2019,
  author  = {Moon, Kevin R. and van Dijk, David and Wang, Zheng and Gigante, Scott
             and Burkhardt, Daniel B. and Chen, William S. and Yim, Kristina
             and van den Elzen, Antonia and Hirn, Matthew J. and Coifman, Ronald R.
             and Ivanova, Natalia B. and Wolf, Guy and Krishnaswamy, Smita},
  title   = {Visualizing structure and transitions in high-dimensional biological data},
  journal = {Nature Biotechnology},
  volume  = {37},
  number  = {12},
  pages   = {1482--1492},
  year    = {2019},
  doi     = {10.1038/s41587-019-0336-3},
}

@article{ollivier2009,
  author  = {Ollivier, Yann},
  title   = {{R}icci curvature of {M}arkov chains on metric spaces},
  journal = {Journal of Functional Analysis},
  volume  = {256},
  number  = {3},
  pages   = {810--864},
  year    = {2009},
  doi     = {10.1016/j.jfa.2008.11.001},
}

@article{forman2003,
  author  = {Forman, Robin},
  title   = {{B}ochner's method for cell complexes and combinatorial {R}icci curvature},
  journal = {Discrete \& Computational Geometry},
  volume  = {29},
  pages   = {323--374},
  year    = {2003},
  doi     = {10.1007/s00454-002-0743-x},
}

@article{linluyau2011,
  author  = {Lin, Yong and Lu, Linyuan and Yau, Shing-Tung},
  title   = {{R}icci curvature of graphs},
  journal = {Tohoku Mathematical Journal},
  volume  = {63},
  number  = {4},
  pages   = {605--627},
  year    = {2011},
  doi     = {10.2748/tmj/1325886283},
}

@article{garciatrillos2023,
  author  = {Garc{\'\i}a Trillos, Nicol{\'a}s and Weber, Melanie},
  title   = {Continuum limits of {O}llivier's {R}icci curvature on data clouds:
             pointwise consistency and global lower bounds},
  journal = {Discrete \& Computational Geometry},
  year    = {2026},
  doi     = {10.1007/s00454-026-00839-5},
  note    = {arXiv:2307.02378},
}

@article{coifman2006,
  author  = {Coifman, Ronald R. and Lafon, St{\'e}phane},
  title   = {Diffusion maps},
  journal = {Applied and Computational Harmonic Analysis},
  volume  = {21},
  number  = {1},
  pages   = {5--30},
  year    = {2006},
  doi     = {10.1016/j.acha.2006.04.006},
}

@article{coifman2005pnas,
  author  = {Coifman, Ronald R. and Lafon, St{\'e}phane and Lee, Ann B.
             and Maggioni, Mauro and Nadler, Boaz and Warner, Frederick
             and Zucker, Steven W.},
  title   = {Geometric diffusions as a tool for harmonic analysis and structure
             definition of data: diffusion maps},
  journal = {Proceedings of the National Academy of Sciences},
  volume  = {102},
  number  = {21},
  pages   = {7426--7431},
  year    = {2005},
  doi     = {10.1073/pnas.0500334102},
}

@article{singer2006,
  author  = {Singer, Amit},
  title   = {From graph to manifold {L}aplacian: the convergence rate},
  journal = {Applied and Computational Harmonic Analysis},
  volume  = {21},
  number  = {1},
  pages   = {128--134},
  year    = {2006},
  doi     = {10.1016/j.acha.2006.03.004},
}

@inproceedings{hein2005,
  author    = {Hein, Matthias and Audibert, Jean-Yves and von Luxburg, Ulrike},
  title     = {From graphs to manifolds: weak and strong pointwise consistency of
               graph {L}aplacians},
  booktitle = {Conference on Learning Theory (COLT)},
  pages     = {470--485},
  year      = {2005},
  doi       = {10.1007/11503415_32},
}

@article{vonluxburg2007,
  author  = {von Luxburg, Ulrike},
  title   = {A tutorial on spectral clustering},
  journal = {Statistics and Computing},
  volume  = {17},
  number  = {4},
  pages   = {395--416},
  year    = {2007},
  doi     = {10.1007/s11222-007-9033-z},
}

@book{bakry2014,
  author    = {Bakry, Dominique and Gentil, Ivan and Ledoux, Michel},
  title     = {Analysis and Geometry of {M}arkov Diffusion Operators},
  series    = {Grundlehren der mathematischen Wissenschaften},
  volume    = {348},
  publisher = {Springer},
  year      = {2014},
  doi       = {10.1007/978-3-319-00227-9},
}

@article{jones2024,
  author  = {Jones, Iolo},
  title   = {Diffusion Geometry},
  journal = {arXiv preprint arXiv:2405.10858},
  year    = {2024},
}

@article{jones2026,
  author  = {Jones, Iolo and Lanners, David},
  title   = {Computing Diffusion Geometry},
  journal = {arXiv preprint arXiv:2602.06006},
  year    = {2026},
}

@inproceedings{ansuini2019,
  author    = {Ansuini, Alessio and Laio, Alessandro and Macke, Jakob H. and Zoccolan, Davide},
  title     = {Intrinsic dimension of data representations in deep neural networks},
  booktitle = {Advances in Neural Information Processing Systems},
  year      = {2019},
}

@article{naitzat2020,
  author  = {Naitzat, Gregory and Zhitnikov, Andrey and Lim, Lek-Heng},
  title   = {Topology of deep neural networks},
  journal = {Journal of Machine Learning Research},
  volume  = {21},
  number  = {184},
  pages   = {1--40},
  year    = {2020},
}

@article{papyan2020,
  author  = {Papyan, Vardan and Han, X. Y. and Donoho, David L.},
  title   = {Prevalence of neural collapse during the terminal phase of deep learning training},
  journal = {Proceedings of the National Academy of Sciences},
  volume  = {117},
  number  = {40},
  pages   = {24652--24663},
  year    = {2020},
  doi     = {10.1073/pnas.2015509117},
}

@inproceedings{zhu2021,
  author    = {Zhu, Zhihui and Ding, Tianyu and Zhou, Jinxin and Li, Xiao
               and You, Chong and Sulam, Jeremias and Qu, Qing},
  title     = {A geometric analysis of neural collapse with unconstrained features},
  booktitle = {Advances in Neural Information Processing Systems},
  year      = {2021},
}

@inproceedings{kornblith2019,
  author    = {Kornblith, Simon and Norouzi, Mohammad and Lee, Honglak and Hinton, Geoffrey},
  title     = {Similarity of neural network representations revisited},
  booktitle = {International Conference on Machine Learning},
  year      = {2019},
}

@inproceedings{alain2017,
  author    = {Alain, Guillaume and Bengio, Yoshua},
  title     = {Understanding intermediate layers using linear classifier probes},
  booktitle = {ICLR Workshop Track},
  year      = {2017},
}

@inproceedings{jacot2018,
  author    = {Jacot, Arthur and Gabriel, Franck and Hongler, Cl{\'e}ment},
  title     = {Neural tangent kernel: convergence and generalization in neural networks},
  booktitle = {Advances in Neural Information Processing Systems},
  year      = {2018},
}

@inproceedings{lee2018,
  author    = {Lee, Jaehoon and Bahri, Yasaman and Novak, Roman and Schoenholz, Samuel S.
               and Pennington, Jeffrey and Sohl-Dickstein, Jascha},
  title     = {Deep neural networks as {G}aussian processes},
  booktitle = {International Conference on Learning Representations},
  year      = {2018},
}

@book{villani2008,
  author    = {Villani, C{\'e}dric},
  title     = {Optimal Transport: Old and New},
  series    = {Grundlehren der mathematischen Wissenschaften},
  volume    = {338},
  publisher = {Springer},
  year      = {2008},
  doi       = {10.1007/978-3-540-71050-9},
}

@article{calder2022lipschitz,
  author  = {Calder, Jeff and Garc{\'\i}a Trillos, Nicol{\'a}s and Lewicka, Marta},
  title   = {{L}ipschitz regularity of graph {L}aplacians on random data clouds},
  journal = {SIAM Journal on Mathematical Analysis},
  volume  = {54},
  number  = {1},
  pages   = {1169--1222},
  year    = {2022},
  doi     = {10.1137/20M1356610},
}

@inproceedings{ting2010,
  author    = {Ting, Daniel and Huang, Ling and Jordan, Michael I.},
  title     = {An analysis of the convergence of graph {L}aplacians},
  booktitle = {International Conference on Machine Learning (ICML)},
  year      = {2010},
}

@article{klus2018transfer,
  author  = {Klus, Stefan and N{\"u}ske, Feliks and Koltai, P{\'e}ter
             and Wu, Hao and Kevrekidis, Ioannis and Sch{\"u}tte, Christof
             and No{\'e}, Frank},
  title   = {Data-Driven Model Reduction and Transfer Operator Approximation},
  journal = {Journal of Nonlinear Science},
  volume  = {28},
  pages   = {985--1010},
  year    = {2018},
  doi     = {10.1007/s00332-017-9437-7}
}

@article{networkmanifold2024,
  title         = {Exploring the Manifold of Neural Networks Using Diffusion Geometry},
  author        = {Abel, Elliott and Steindl, Andrew J. and Mazioud, Selma and Schueler, Ellie and Ogundipe, Folu and Zhang, Ellen and Grinspan, Yvan and Reimann, Kristof and Crevasse, Peyton and Bhaskar, Dhananjay and Viswanath, Siddharth and Zhang, Yanlei and Rudner, Tim G. J. and Adelstein, Ian and Krishnaswamy, Smita},
  journal       = {arXiv preprint arXiv:2411.12626},
  year          = {2024},
  eprint        = {2411.12626},
  archivePrefix = {arXiv},
  primaryClass  = {cs.LG},
  doi           = {10.48550/arXiv.2411.12626}
}

@inproceedings{dse2023,
  title     = {Assessing Neural Network Representations During Training Using Noise-Resilient Diffusion Spectral Entropy},
  author    = {Liao, Danqi and Liu, Chen and Christensen, Benjamin W. and Tong, Alexander and Huguet, Guillaume and Wolf, Guy and Nickel, Maximilian and Adelstein, Ian and Krishnaswamy, Smita},
  booktitle = {ICML 2023 Workshop on Topology, Algebra, and Geometry in Machine Learning},
  year      = {2023},
  eprint    = {2312.04823},
  archivePrefix = {arXiv},
  primaryClass  = {cs.CV},
  doi       = {10.48550/arXiv.2312.04823}
}

@inproceedings{yanghu2021,
  author    = {Yang, Greg and Hu, Edward J.},
  title     = {Tensor Programs IV: Feature Learning in Infinite-Width Neural Networks},
  booktitle = {Proceedings of the 38th International Conference on Machine Learning},
  series    = {Proceedings of Machine Learning Research},
  volume    = {139},
  pages     = {11727--11737},
  year      = {2021},
  publisher = {PMLR}
}

@inproceedings{mendelman2025,
  author    = {Mendelman, Harel and Talmon, Ronen},
  title     = {Supervised and Semi-Supervised Diffusion Maps with Label-Driven Diffusion},
  booktitle = {International Conference on Learning Representations},
  year      = {2025},
  url       = {https://openreview.net/forum?id=G3B5ReApDw}
}

@article{muandet2017,
  author  = {Muandet, Krikamol and Fukumizu, Kenji and Sriperumbudur, Bharath K. and Sch{\"o}lkopf, Bernhard},
  title   = {Kernel Mean Embedding of Distributions: A Review and Beyond},
  journal = {Foundations and Trends in Machine Learning},
  volume  = {10},
  number  = {1--2},
  pages   = {1--141},
  year    = {2017},
  doi     = {10.1561/2200000060}
}

@article{jernigan2003,
  author  = {Jernigan, Robert W. and Baran, Robert H.},
  title   = {Testing Lumpability in Markov Chains},
  journal = {Statistics \& Probability Letters},
  volume  = {64},
  number  = {1},
  pages   = {17--23},
  year    = {2003},
  doi     = {10.1016/S0167-7152(03)00126-3}
}

@techreport{krizhevsky2009,
  author      = {Krizhevsky, Alex},
  title       = {Learning Multiple Layers of Features from Tiny Images},
  institution = {University of Toronto},
  year        = {2009}
}

@inproceedings{he2016,
  author    = {He, Kaiming and Zhang, Xiangyu and Ren, Shaoqing and Sun, Jian},
  title     = {Deep Residual Learning for Image Recognition},
  booktitle = {Proceedings of the IEEE Conference on Computer Vision and Pattern Recognition},
  pages     = {770--778},
  year      = {2016},
  doi       = {10.1109/CVPR.2016.90}
}

@book{kemeny1976,
  author    = {Kemeny, John G. and Snell, J. Laurie},
  title     = {Finite Markov Chains},
  series    = {Undergraduate Texts in Mathematics},
  publisher = {Springer-Verlag},
  address   = {New York},
  year      = {1976},
  isbn      = {978-0-387-90192-3}
}
\bibliographystyle{tmlr}

\appendix
\section{Proofs for the Gaussian affinity reduction}
\label{app:gaussian-bridge}

We now derive Theorem~\ref{thm:gaussian-bridge},
Proposition~\ref{prop:population-leakage}, and
Corollary~\ref{cor:homogeneous}, followed by the kernel-tilted local proxies in
Appendix~\ref{app:local-geometry}. 

Assume a balanced Gaussian snapshot model with shared covariance
\[
z\mid y=a \sim \mathcal N(\mu_a,\Sigma),
\qquad
a\in[K],
\qquad
\mathbb P(y=a)=\frac1K,
\]
and use the Gaussian kernel
\[
k_\varepsilon(u,v)
=
\exp\!\left(
-\frac{\|u-v\|^2}{4\varepsilon}
\right).
\]
For general class-conditionals \(\nu_a\), define
\[
\alpha_{ab}^{(\varepsilon)}
=
\mathbb E_{X\sim\nu_a,\;X'\sim\nu_b}
\left[
k_\varepsilon(X,X')
\right].
\]
The Gaussian calculation below is the special case in which these
population affinities can be written explicitly in terms of class means and
a shared covariance. In the balanced shared-covariance Gaussian model case,
\(\alpha_{ab}^{(\varepsilon)}=\alpha_0e^{-c_\varepsilon^{(a,b)}}\), so the
overlap chain has the closed form in Theorem~\ref{thm:gaussian-bridge}.
Proposition~\ref{prop:degree-heterogeneity} separately controls its distance
from the exact row-normalized population transition.

For two classes \(a,b\), write
\[
\Delta^{(a,b)}:=\mu_b-\mu_a
\]
and
\[
c_\varepsilon^{(a,b)}
=
\frac14
(\mu_a-\mu_b)^\top
(\varepsilon I+\Sigma)^{-1}
(\mu_a-\mu_b).
\]
Where \(c_\varepsilon^{(a,a)}=0\).

\subsection{Proof of Theorem~\ref{thm:gaussian-bridge}}
\label{app:proof-gaussian-bridge}

Let
\[
X_a\sim\mathcal N(\mu_a,\Sigma),
\qquad
X_b\sim\mathcal N(\mu_b,\Sigma)
\]
be independent. Then
\[
X_a-X_b
\sim
\mathcal N(\mu_a-\mu_b,2\Sigma).
\]
We need to compute
\[
\alpha_{ab}
=
\mathbb E
\exp\!\left(
-\frac{\|X_a-X_b\|^2}{4\varepsilon}
\right).
\]

Diagonalize \(\Sigma=U\Lambda U^\top\), where
\(\Lambda=\mathrm{diag}(\lambda_1,\dots,\lambda_d)\), and define
\[
\widetilde\Delta^{(a,b)}
=
U^\top(\mu_a-\mu_b).
\]
Since the kernel is rotation invariant, the expectation factorizes over
coordinates. If \(G\sim\mathcal N(m,s^2)\), then
\[
\mathbb E[e^{-tG^2}]
=
(1+2ts^2)^{-1/2}
\exp\!\left(
-\frac{tm^2}{1+2ts^2}
\right).
\]
Applying this with
\[
t=\frac{1}{4\varepsilon},
\qquad
s^2=2\lambda_r,
\qquad
m=\widetilde\Delta^{(a,b)}_r,
\]
gives
\[
\mathbb E
\exp\!\left(
-\frac{G^2}{4\varepsilon}
\right)
=
\left(1+\frac{\lambda_r}{\varepsilon}\right)^{-1/2}
\exp\!\left(
-\frac{(\widetilde\Delta^{(a,b)}_r)^2}
{4(\varepsilon+\lambda_r)}
\right).
\]
Multiplying over coordinates yields
\[
\alpha_{ab}
=
\det\!\left(I+\frac{\Sigma}{\varepsilon}\right)^{-1/2}
\exp\!\left(
-\frac14
(\mu_a-\mu_b)^\top
(\varepsilon I+\Sigma)^{-1}
(\mu_a-\mu_b)
\right).
\]
Thus
\[
\alpha_{ab}
=
\alpha_0
\exp\!\left(-c_\varepsilon^{(a,b)}\right),
\qquad
\alpha_0
=
\det\!\left(I+\frac{\Sigma}{\varepsilon}\right)^{-1/2}.
\]

Since \(\alpha_0\) is common to every pair of classes, it cancels under
row-normalization. Therefore the Gaussian overlap-chain transition matrix
is
\[
\bar P_{ab}
=
\frac{\alpha_{ab}}
{\sum_{r=1}^K \alpha_{ar}}
=
\frac{
\exp\!\left(-c_\varepsilon^{(a,b)}\right)
}{
\sum_{r=1}^K
\exp\!\left(-c_\varepsilon^{(a,r)}\right)
}.
\]

\paragraph{Special cases.}
If \(\Sigma=\sigma^2 I\), then
\[
c_\varepsilon^{(a,b)}
=
\frac{\|\mu_a-\mu_b\|^2}
{4(\varepsilon+\sigma^2)}.
\]
If \(\Sigma=\sigma^2 I+\beta vv^\top\), with \(\|v\|=1\), and
\[
\mu_a-\mu_b=\Delta_\parallel v+\Delta_\perp,
\qquad
\Delta_\perp\perp v,
\]
then
\[
c_\varepsilon^{(a,b)}
=
\frac{\Delta_\parallel^2}
{4(\varepsilon+\sigma^2+\beta)}
+
\frac{\|\Delta_\perp\|^2}
{4(\varepsilon+\sigma^2)}.
\]
If \(\Sigma\) is invertible and \(\varepsilon\to 0\), then
\[
c_\varepsilon^{(a,b)}
\to
\frac14
(\mu_a-\mu_b)^\top
\Sigma^{-1}
(\mu_a-\mu_b).
\]

\subsection{Stationary distribution and leakage}
\label{app:coarse-reversibility}
\label{app:proof-population-leakage}

Define
\[
w_{ab}
=
\exp\!\left(-c_\varepsilon^{(a,b)}\right),
\qquad
q_a=\sum_{r=1}^K w_{ar},
\qquad
\bar P_{ab}=\frac{w_{ab}}{q_a}.
\]
Since $w_{ab}=w_{ba}$, the overlap chain is reversible with stationary
distribution
\[
\bar\pi_a
=
\frac{q_a}{\sum_{s=1}^Kq_s},
\]
because
\[
\bar\pi_a\bar P_{ab}
=
\frac{w_{ab}}{\sum_s q_s}
=
\bar\pi_b\bar P_{ba}.
\]

Its stationary cross-class leakage is therefore
\[
\mathcal L_{\mathrm{stat}}(T^{\mathrm{ov}})
=
\sum_a\bar\pi_a\sum_{b\neq a}\bar P_{ab}
=
\frac{\sum_{a\neq b}w_{ab}}{\sum_s q_s}.
\]
Since $w_{aa}=1$,
\[
\sum_s q_s
=
K+\sum_{a\neq b}w_{ab},
\]
and hence
\[
\mathcal L_{\mathrm{stat}}(T^{\mathrm{ov}})
=
\frac{
\sum_{a\neq b}\exp\!\left(-c_\varepsilon^{(a,b)}\right)
}{
K+\sum_{a\neq b}\exp\!\left(-c_\varepsilon^{(a,b)}\right)
}.
\]
This proves Proposition~\ref{prop:population-leakage}.

\subsection{Proof of Corollary~\ref{cor:homogeneous}}
\label{app:proof-homogeneous}

Suppose
\[
c_\varepsilon^{(a,b)}=c
\qquad
\text{for every }a\neq b,
\]
and $t=e^{-c}$. The overlap chain here is
\[
\bar P
=
\frac{1-t}{1+(K-1)t}I
+
\frac{t}{1+(K-1)t}\mathbf 1\mathbf 1^\top.
\]
Its diagonal entry is $1/(1+(K-1)t)$, so the stationary leakage is
\[
\mathcal L_{\mathrm{stat}}(T^{\mathrm{ov}})
=
\frac{(K-1)t}{1+(K-1)t}
=
\frac{(K-1)e^{-c}}{1+(K-1)e^{-c}}.
\]

The vector $\mathbf 1$ has eigenvalue $1$, while every vector orthogonal to
$\mathbf 1$ has eigenvalue
\[
\lambda_2=\cdots=\lambda_K
=
\frac{1-t}{1+(K-1)t}
=
\frac{e^c-1}{e^c+K-1}.
\]
Hence the spectral gap is
\[
1-\lambda_2
=
\frac{Kt}{1+(K-1)t}
=
\frac{K}{e^c+K-1}.
\]

\subsection{Kernel-tilted local proxies}
\label{app:local-geometry}

Coarse class affinities, the overlap chain, and stationary leakage depend on
the Gaussian snapshot only through the scalars
\(c_\varepsilon^{(a,b)}\), whereas local operator geometry also retains the
directions of the covariance and class-mean differences. To make this
distinction explicit, let
\[
\delta=X_b-X_a,
\qquad
X_a\sim\mathcal N(\mu_a,\Sigma),
\qquad
X_b\sim\mathcal N(\mu_b,\Sigma),
\]
so that
\[
\delta\sim
\mathcal N\!\left(\Delta^{(a,b)},2\Sigma\right),
\qquad
\Delta^{(a,b)}=\mu_b-\mu_a.
\]
Reweighting this displacement by
\(\exp(-\|\delta\|^2/(4\varepsilon))\) gives a Gaussian tilted law with mean
and covariance
\[
m_\varepsilon^{(a,b)}
=
\left(I+\frac{\Sigma}{\varepsilon}\right)^{-1}
\Delta^{(a,b)},
\qquad
V_\varepsilon
=
2\Sigma
\left(I+\frac{\Sigma}{\varepsilon}\right)^{-1}.
\]

Define the overlap-averaged coordinate second-moment proxy for source class
$a$ by
\[
G_a^{\mathrm{ov}}
:=
\frac{1}{2\varepsilon\bar s_a}
\sum_{b=1}^K
\pi_b
\mathbb E\!\left[
k_\varepsilon(X_a,X_b)
(X_b-X_a)(X_b-X_a)^\top
\right].
\]
Under the shared-covariance Gaussian model, the tilted moments above give
\[
G_a^{\mathrm{ov}}
=
\frac{1}{2\varepsilon}V_\varepsilon
+
\frac{1}{2\varepsilon}
\sum_{b=1}^K
T^{\mathrm{ov}}_{ab}
m_\varepsilon^{(a,b)}
m_\varepsilon^{(a,b)\top},
\]
or equivalently
\[
G_a^{\mathrm{ov}}
=
\frac{1}{\varepsilon}
\Sigma
\left(I+\frac{\Sigma}{\varepsilon}\right)^{-1}
+
\frac{1}{2\varepsilon}
\sum_{b=1}^K
T^{\mathrm{ov}}_{ab}
m_\varepsilon^{(a,b)}
m_\varepsilon^{(a,b)\top}.
\]
Unlike the coarse class quantities, this proxy depends on the orientation of
each $\Delta^{(a,b)}$ and on the eigendirections of $\Sigma$.

The corresponding overlap-averaged squared-radius proxy is
\[
\left(\rho_{a,\varepsilon}^{\mathrm{ov}}\right)^2
:=
\frac{1}{\bar s_a}
\sum_{b=1}^K
\pi_b
\mathbb E\!\left[
k_\varepsilon(X_a,X_b)\|X_b-X_a\|^2
\right].
\]
It has the closed form
\[
\left(\rho_{a,\varepsilon}^{\mathrm{ov}}\right)^2
=
2\,\mathrm{tr}\!\left(
\Sigma
\left(I+\frac{\Sigma}{\varepsilon}\right)^{-1}
\right)
+
\sum_{b=1}^K
T^{\mathrm{ov}}_{ab}
\left\|
\left(I+\frac{\Sigma}{\varepsilon}\right)^{-1}
\Delta^{(a,b)}
\right\|^2.
\]

The exact class-averaged population operator quantities retain the
source-dependent denominator $s(X)$ inside the expectation. They coincide
with these overlap proxies when $s(X)=\bar s_a$ almost surely within source
class $a$. 

\subsection{Beyond the shared-covariance Gaussian model}
\label{app:beyond-gaussian}

Gaussianity is needed only for the closed-form formulas. The operator and
population definitions themselves apply to arbitrary class-conditional laws.

If classes have unequal covariances,
\[
z\mid y=a\sim\mathcal N(\mu_a,\Sigma_a),
\]
then kernel affinities remain closed form, since
\[
X_a-X_b
\sim
\mathcal N(\mu_a-\mu_b,\Sigma_a+\Sigma_b).
\]
In that case,
\[
\mathbb E[k_\varepsilon(X_a,X_b)]
=
\det\!\left(
I+\frac{\Sigma_a+\Sigma_b}{2\varepsilon}
\right)^{-1/2}
\exp\!\left(
-\frac{1}{4\varepsilon}
(\mu_a-\mu_b)^\top
\left(
I+\frac{\Sigma_a+\Sigma_b}{2\varepsilon}
\right)^{-1}
(\mu_a-\mu_b)
\right).
\]
Unlike the shared-covariance case, the determinant factor now depends on
\((a,b)\), so the affinity cannot be summarized by a single shared
multiplicative constant times \(e^{-c_\varepsilon^{(a,b)}}\).

For non-Gaussian or multimodal class-conditionals, the same operator
observables are still well-defined, but the reduction to means and
covariances need not hold. A useful population analogue of the Gaussian separation is the symmetric
operator-native quantity
\[
\widetilde c_\varepsilon^{(a,b)}
:=
-\log
\frac{
\mathbb E[k_\varepsilon(z,z')\mid y=a,\ y'=b]
}{
\sqrt{
\mathbb E[k_\varepsilon(z,z')\mid y=a,\ y'=a]\,
\mathbb E[k_\varepsilon(z,z')\mid y=b,\ y'=b]
}
}.
\]
In the shared-covariance Gaussian model, the within-class affinities are
equal and the cross-class affinity is
\(\alpha_{ab}=\alpha_0 e^{-c_\varepsilon^{(a,b)}}\). Hence
\[
\widetilde c_\varepsilon^{(a,b)}
=
-\log
\frac{\alpha_0 e^{-c_\varepsilon^{(a,b)}}}{\alpha_0}
=
c_\varepsilon^{(a,b)}.
\]
For non-Gaussian classes,
\(\widetilde c_\varepsilon^{(a,b)}\) still measures affinity in the kernel
geometry, even though no Mahalanobis reduction is available.

\section{Proofs for stability and hard-neighborhood discontinuity}
\label{app:stability}

We prove Theorem~\ref{thm:operator-stability},
Corollary~\ref{cor:observable-stability}, and
Proposition~\ref{prop:knn-discontinuity}. Write
\[
Z=(z_1,\dots,z_n)\in(\mathbb R^d)^n,
\qquad
\widetilde Z=(\widetilde z_1,\dots,\widetilde z_n)\in(\mathbb R^d)^n,
\]
and define
\[
\|Z-\widetilde Z\|_\infty
=
\max_{1\le i\le n}\|z_i-\widetilde z_i\|.
\]
For fixed \(\varepsilon>0\), define
\[
W_{ij}(Z)
=
\exp\!\left(
-\frac{\|z_i-z_j\|^2}{4\varepsilon}
\right),
\qquad
D_i(Z)=\sum_{j=1}^n W_{ij}(Z),
\qquad
P_{ij}(Z)=\frac{W_{ij}(Z)}{D_i(Z)}.
\]
For matrices we use the row-sum operator norm
\[
\|A\|_{\infty\to\infty}
=
\max_i\sum_j |A_{ij}|.
\]

\subsection{Distance and kernel perturbation bounds}
\label{app:distance-kernel-bounds}

\begin{lemma}[Distance and kernel perturbations]
\label{lem:distance-perturbation}
\label{lem:kernel-perturbation}
Assume
\[
\max_i\|z_i\|\le R,
\qquad
\max_i\|\widetilde z_i\|\le R,
\]
and let
\[
\eta=\|Z-\widetilde Z\|_\infty.
\]
Then, for every $i,j$,
\[
\left|
\|z_i-z_j\|^2
-
\|\widetilde z_i-\widetilde z_j\|^2
\right|
\le 8R\eta
\]
and
\[
|W_{ij}(Z)-W_{ij}(\widetilde Z)|
\le
\frac{2R}{\varepsilon}\eta.
\]
\end{lemma}

\begin{proof}
Set
\[
a=z_i-z_j,
\qquad
b=\widetilde z_i-\widetilde z_j.
\]
Then
\[
\big|\|a\|^2-\|b\|^2\big|
=
|\langle a+b,a-b\rangle|
\le
\|a+b\|\,\|a-b\|
\le
(4R)(2\eta)
=
8R\eta.
\]
The derivative of $u\mapsto e^{-u/(4\varepsilon)}$ has absolute value at
most $1/(4\varepsilon)$ for $u\ge0$, so the mean-value theorem gives
\[
|W_{ij}(Z)-W_{ij}(\widetilde Z)|
\le
\frac{1}{4\varepsilon}(8R\eta)
=
\frac{2R}{\varepsilon}\eta.
\]
\end{proof}

\subsection{Proof of Theorem~\ref{thm:operator-stability}}
\label{app:proof-operator-stability}

Let
\[
\eta=\|Z-\widetilde Z\|_\infty,
\qquad
L_W=\frac{2R}{\varepsilon}.
\]
By Lemma~\ref{lem:kernel-perturbation},
\[
|W_{ij}(Z)-W_{ij}(\widetilde Z)|
\le
L_W\eta
\]
for every \(i,j\).

Because all points lie in the radius-\(R\) ball, every pairwise distance is
at most \(2R\). Hence
\[
W_{ij}(Z)\ge
\exp(-R^2/\varepsilon)
=: \beta,
\qquad
W_{ij}(\widetilde Z)\ge \beta.
\]
Consequently,
\[
D_i(Z)\ge n\beta,
\qquad
D_i(\widetilde Z)\ge n\beta.
\]

Fix a row \(i\), and write
\[
w_j=W_{ij}(Z),
\qquad
\widetilde w_j=W_{ij}(\widetilde Z),
\qquad
D=\sum_j w_j,
\qquad
\widetilde D=\sum_j \widetilde w_j.
\]
Let
\[
\Delta_i
=
\sum_{j=1}^n |w_j-\widetilde w_j|.
\]
Then
\[
\Delta_i\le nL_W\eta,
\qquad
|D-\widetilde D|\le \Delta_i.
\]
Now
\[
\begin{aligned}
\sum_j
\left|
\frac{w_j}{D}
-
\frac{\widetilde w_j}{\widetilde D}
\right|
&\le
\sum_j
\frac{|w_j-\widetilde w_j|}{D}
+
\sum_j
\widetilde w_j
\left|
\frac1D-\frac1{\widetilde D}
\right|  \\
&=
\frac{\Delta_i}{D}
+
\widetilde D
\frac{|D-\widetilde D|}{D\widetilde D} \\
&\le
\frac{2\Delta_i}{D}.
\end{aligned}
\]
Using \(D\ge n\beta\), we get
\[
\sum_j
|P_{ij}(Z)-P_{ij}(\widetilde Z)|
\le
\frac{2nL_W\eta}{n\beta}
=
\frac{2L_W}{\beta}\eta.
\]
Taking the maximum over rows yields
\[
\|P_\varepsilon(Z)-P_\varepsilon(\widetilde Z)\|_{\infty\to\infty}
\le
\frac{2L_W}{\beta}
\|Z-\widetilde Z\|_\infty.
\]
Substituting
\[
L_W=\frac{2R}{\varepsilon},
\qquad
\beta=e^{-R^2/\varepsilon},
\]
gives
\[
\|P_\varepsilon(Z)-P_\varepsilon(\widetilde Z)\|_{\infty\to\infty}
\le
\frac{4R}{\varepsilon}
\exp(R^2/\varepsilon)
\|Z-\widetilde Z\|_\infty.
\]
This proves Theorem~\ref{thm:operator-stability}. Since the entrywise
maximum norm is bounded by the row-sum norm, the same estimate also
controls individual entries.

\subsection{Proof of Corollary~\ref{cor:observable-stability}}
\label{app:proof-observable-stability}

Let
\[
\eta=\|Z-\widetilde Z\|_\infty,
\qquad
C_P
=
\frac{4R}{\varepsilon}\exp(R^2/\varepsilon).
\]
Theorem~\ref{thm:operator-stability} gives
\[
\|P(Z)-P(\widetilde Z)\|_{\infty\to\infty}
\le C_P\eta.
\]

Consider any statistic of the form
\[
\Phi(P)
=
\sum_i\omega_i\sum_j a_{ij}P_{ij},
\]
where $\omega_i\ge0$, $\sum_i\omega_i=1$, and
$|a_{ij}|\le M$. Then
\[
\begin{aligned}
|\Phi(P)-\Phi(\widetilde P)|
&\le
\sum_i\omega_i\sum_j
|a_{ij}|\,|P_{ij}-\widetilde P_{ij}| \\
&\le
M\|P-\widetilde P\|_{\infty\to\infty}
\le
MC_P\eta.
\end{aligned}
\]
The aggregate transition entry $T^{\mathrm{agg}}_{ab}$ and raw cross-class
leakage are obtained with $M=1$, and are therefore $C_P$-Lipschitz. For a
class indicator $g_a$, the boundary energy uses
\[
a_{ij}
=
\frac{(g_a(j)-g_a(i))^2}{2\varepsilon},
\]
giving Lipschitz constant $C_P/(2\varepsilon)$. The total one-hot
label-boundary energy satisfies
\[
\widehat{\mathcal E}_{\mathrm{label}}(Z)
=
\frac{1}{\varepsilon n}
\sum_i\sum_{j:y_j\neq y_i}P_{ij}(Z),
\]
and hence has Lipschitz constant $C_P/\varepsilon$.

For the squared soft diffusion radius,
\[
\widehat\rho_\varepsilon^2(i;Z)
=
\sum_jP_{ij}(Z)\|z_j-z_i\|^2,
\]
adding and subtracting
\(\sum_jP_{ij}(\widetilde Z)\|z_j-z_i\|^2\) gives
\[
\begin{aligned}
\left|
\widehat\rho_\varepsilon^2(i;Z)
-
\widehat\rho_\varepsilon^2(i;\widetilde Z)
\right|
&\le
\sum_j
|P_{ij}(Z)-P_{ij}(\widetilde Z)|
\|z_j-z_i\|^2 \\
&\quad+
\sum_jP_{ij}(\widetilde Z)
\left|
\|z_j-z_i\|^2
-
\|\widetilde z_j-\widetilde z_i\|^2
\right|.
\end{aligned}
\]
Since $\|z_j-z_i\|^2\le4R^2$, Lemma~\ref{lem:distance-perturbation} yields
\[
\left|
\widehat\rho_\varepsilon^2(i;Z)
-
\widehat\rho_\varepsilon^2(i;\widetilde Z)
\right|
\le
\left(4R^2C_P+8R\right)\eta.
\]
The root radius is locally Lipschitz wherever the squared radius is bounded
away from zero. Combining these bounds proves
Corollary~\ref{cor:observable-stability}.

\subsection{Removing self-loops in finite-sample transport}
\label{app:no-self-loop-transport}

The main theoretical operator \(P_\varepsilon(Z)\) includes self-loops
because \(W_{ii}=1\). In finite empirical class-transport calculations, we
also use row normalization without self-loops
\[
Q_{ij}(Z)
=
\frac{P_{ij}(Z)\mathbf 1\{i\neq j\}}
{\sum_{m\neq i}P_{im}(Z)}
\]
whenever the denominator is nonzero. This matches the convention used in
the empirical aggregate transition computations.

Let
\[
s_i(Z)=\sum_{m\neq i}P_{im}(Z).
\]
On any subset of feature clouds where
\[
s_i(Z)\ge \gamma>0
\qquad
\text{for all }i,
\]
the map \(P\mapsto Q\) is Lipschitz in row-sum norm. Indeed, for each row,
the same normalization argument used in the proof of
Theorem~\ref{thm:operator-stability} gives
\[
\sum_j |Q_{ij}(Z)-Q_{ij}(\widetilde Z)|
\le
\frac{2}{\gamma}
\sum_j |P_{ij}(Z)-P_{ij}(\widetilde Z)|.
\]
Therefore the aggregate transition (without self loops) and its derived observables
are also Lipschitz on such subsets. The condition \(s_i(Z)\ge\gamma\)
excludes the degenerate finite sample regime in which a row is essentially
all self loop mass.

\subsection{Discontinuity of hard neighborhood graphs}
\label{app:proof-knn-discontinuity}

We now prove Proposition~\ref{prop:knn-discontinuity}. Let
\(A_k(Z)\) denote the adjacency matrix of a directed \(k\)-nearest-neighbor
graph, with an arbitrary but fixed deterministic tie-breaking rule. The
same argument applies to symmetrized \(k\)-nearest-neighbor graphs whenever
the changed directed edge affects the symmetrized adjacency.

Fix \(1\le k<n-1\). Suppose there exists a point \(z_i\) for which the
\(k\)-th and \((k+1)\)-st nearest-neighbor distances are equal. That is,
there are two distinct candidate neighbors \(z_p\) and \(z_q\) such that
\[
\|z_i-z_p\|
=
\|z_i-z_q\|
=
r,
\]
and exactly \(k-1\) other points are strictly closer to \(z_i\).

Assume that the tie-breaking rule selects \(p\) rather than \(q\), so that
\[
A_k(Z)_{ip}=1,
\qquad
A_k(Z)_{iq}=0.
\]
For any \(\eta>0\), we can perturb \(z_p\) and \(z_q\) by less than
\(\eta\) so that \(z_q\) becomes slightly closer to \(z_i\) than \(z_p\),
while all other neighbor orderings remain unchanged. Denote the perturbed
cloud by \(Z^{(\eta)}\). Then
\[
\|Z^{(\eta)}-Z\|_\infty<\eta,
\]
but
\[
A_k(Z^{(\eta)})_{ip}=0,
\qquad
A_k(Z^{(\eta)})_{iq}=1.
\]
Thus an adjacency entry changes by one under an arbitrarily small
perturbation. Hence \(Z\mapsto A_k(Z)\) is discontinuous at \(Z\).

If a graph observable \(M(A_k(Z))\) takes different values on these two
adjacency patterns, then
\[
M(A_k(Z^{(\eta)}))\not\to M(A_k(Z))
\]
as \(\eta\to0\). Thus any adjacency-dependent observable that nontrivially
distinguishes the two graphs is also discontinuous at such a
neighbor-swap configuration.

\section{Empirical degree heterogeneity and aggregation residuals}
\label{app:empirical-degree-lumpability}

Table~\ref{tab:degree-lumpability} reports both reduction diagnostics at the
finest and median bandwidths. The within-class degree coefficient of variation
is the sample counterpart of the control term in
Proposition~\ref{prop:degree-heterogeneity}, whereas the aggregation residual
measures departure from label lumpability. 
\begin{table}[h]
\centering
\small
\setlength{\tabcolsep}{4.5pt}
\begin{tabular}{llcccc}
\toprule
Dataset & Scale & CV, stem & CV, layer 4 & $R_{\mathrm{lump}}$, stem & $R_{\mathrm{lump}}$, layer 4\\
\midrule
CIFAR-10 & $0.25\varepsilon_{\mathrm{med}}$ & $0.598$ & $0.101$ & $0.365$ & $0.224$\\
CIFAR-10 & $\varepsilon_{\mathrm{med}}$ & $0.340$ & $0.076$ & $0.168$ & $0.063$\\
CIFAR-100 & $0.25\varepsilon_{\mathrm{med}}$ & $0.631$ & $0.618$ & $0.478$ & $0.256$\\
CIFAR-100 & $\varepsilon_{\mathrm{med}}$ & $0.326$ & $0.212$ & $0.163$ & $0.038$\\
\bottomrule
\end{tabular}
\caption{Degree heterogeneity and aggregation residuals, averaged over three
trained models. Smaller values indicate more homogeneous kernel degree and a
closer class-level description of one-step transport.}
\label{tab:degree-lumpability}
\end{table}

\section{Matched class-chain perturbation response}
\label{app:matched-stability}

For each method, we form its perturbed and unperturbed $K\times K$ class
matrices and compute
\[
\Delta_{\mathrm{class}}
=
\frac{1}{2K}\sum_{a=1}^K
\left\|T_{a\cdot}(\widetilde Z)-T_{a\cdot}(Z)\right\|_1.
\]
The same uniform source class weighting is used for both methods. The operator
chain uses the diffusion rows from
Section~\ref{sec:operator-snapshots} without self loops. The graph chain assigns each directed
neighbor probability $1/k$ before class aggregation.

\begin{figure*}[h]
\centering
\begin{minipage}{0.49\textwidth}
\centering
\includegraphics[width=\linewidth]{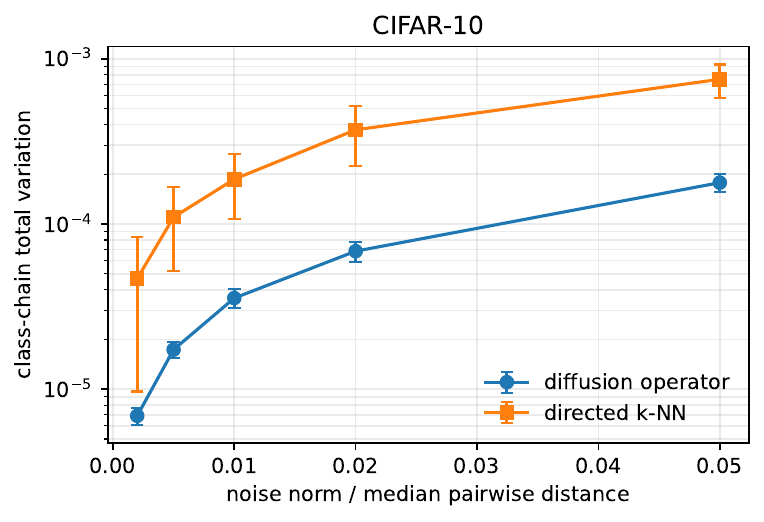}
\end{minipage}\hfill
\begin{minipage}{0.49\textwidth}
\centering
\includegraphics[width=\linewidth]{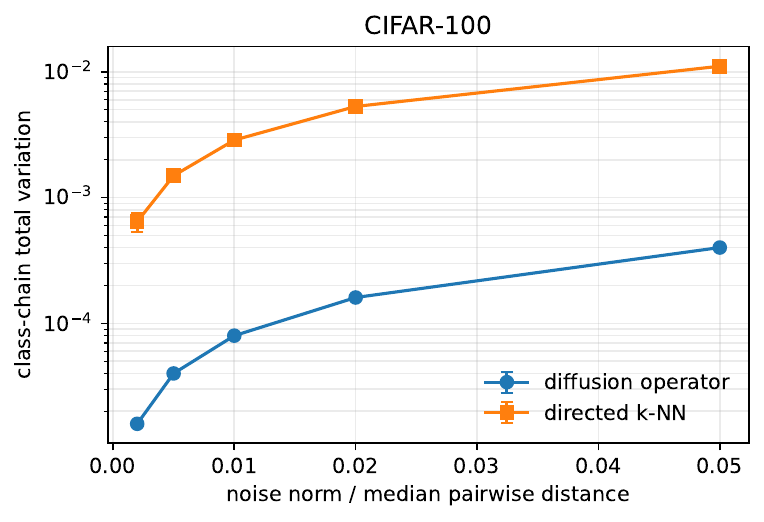}
\end{minipage}
\caption{Matched class-chain total variation under layer-4 feature
perturbations. Error bars show one standard deviation across the 30
model--noise realizations at each nonzero scale. The vertical axis is
logarithmic.}
\label{fig:matched-stability}
\end{figure*}

\section{Separate perturbation mechanisms}
\label{app:stability-mechanisms}

Figure~\ref{fig:stability-mechanisms} separates the mechanisms underlying the matched class-chain
comparison. Operator row total variation measures the smooth movement of
probability mass, directed edge rewiring counts changes in the hard
neighborhoods.

The theoretical bounds are conservative relative to the observed changes.
Realized directed edge-rewiring fractions remain between $0.17\%$ and
$3.7\%$ on CIFAR-10 and between $0.07\%$ and $1.5\%$ on CIFAR-100, while
the operator rows vary continuously throughout the tested perturbation
range.

\begin{figure}[h]
\centering
\includegraphics[width=.9\linewidth]{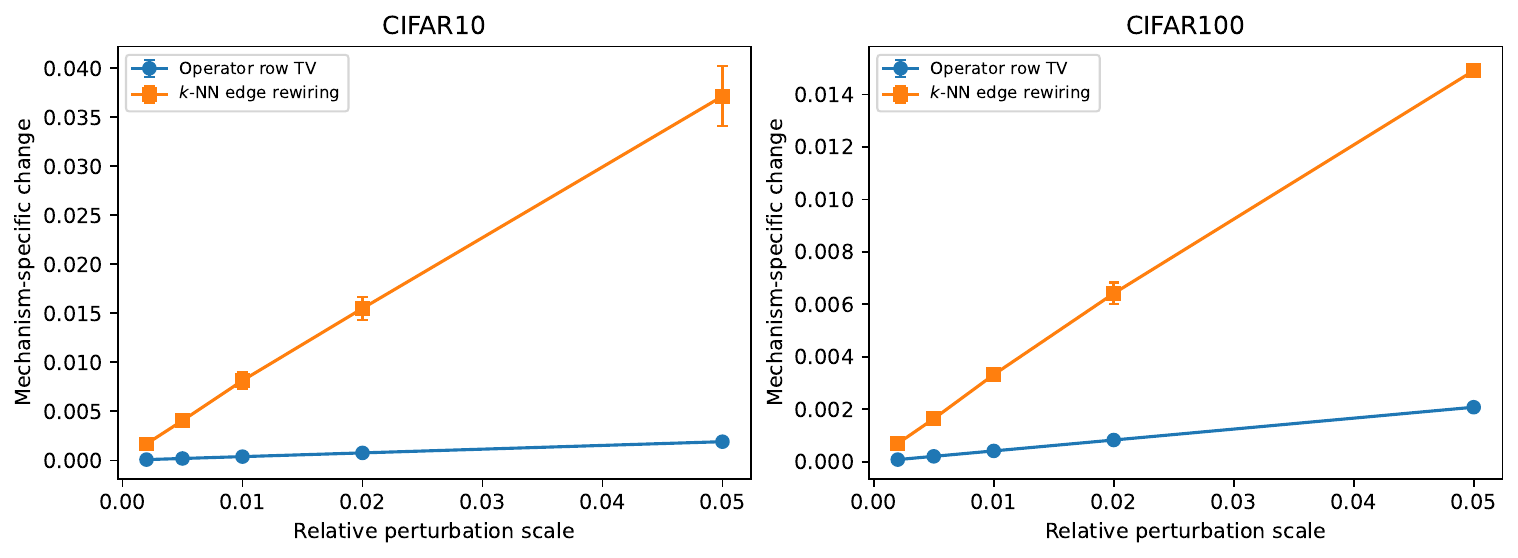}
\caption{Perturbation mechanisms at layer 4. Operator change is mean row total
variation. Graph change is the directed edge-rewiring fraction. Error bars show
one standard deviation across model seeds.}
\label{fig:stability-mechanisms}
\end{figure}

\section{Density-normalized diffusion operators}
\label{app:normalization}

Let
\[
q(x)
=
\sum_{c=1}^K
\pi_c
\int k_\varepsilon(x,z)\,d\nu_c(z)
\]
be the population kernel degree. After the source factor
$q(x)^{-\alpha}$ cancels under row normalization, define the
density-adjusted class affinity
\[
B_{ab}^{(\alpha)}
=
\mathbb E_{X\sim\nu_a,\,X'\sim\nu_b}
\left[
k_\varepsilon(X,X')q(X')^{-\alpha}
\right].
\]
The associated affinity-level overlap chain is
\[
T_{ab}^{(\mathrm{ov},\alpha)}
=
\frac{\pi_bB_{ab}^{(\alpha)}}
{\sum_r\pi_rB_{ar}^{(\alpha)}}.
\]

A simpler classwise-degree approximation is obtained by defining
\[
d_b=\sum_c\pi_cA_{bc}
\]
and replacing $q(X')$ by $d_b$ for $X'\sim\nu_b$. This gives
\[
T_{ab}^{(\mathrm{cw},\alpha)}
=
\frac{\pi_bA_{ab}d_b^{-\alpha}}
{\sum_r\pi_rA_{ar}d_r^{-\alpha}}.
\]
The approximation is exact when the population kernel degree is constant
within each destination class. At $\alpha=0$, both constructions reduce to
the original overlap chain.

The CIFAR experiments apply density normalization directly to the sample
operator for $\alpha\in\{0,\tfrac12,1\}$; the classwise formula above is used
only to describe its affinity-level interpretation.

\section{Controlled calibration experiments}
\label{app:controlled-evaluation}

The CIFAR experiments provide the main empirical evidence, while the Gaussian simulation
tests the affinity formula in the setting where it is exact.

We report mean pairwise separation
$\bar c_\varepsilon$, stationary class-chain leakage, the coarse gap
$1-\operatorname{SLEM}(\widehat P^{(K)})$, as well as root-mean-square soft diffusion radius.

Empirical class transport removes point self-loops and renormalizes rows
before aggregation, whereas the Gaussian overlap formulas predict the chain
formed by normalizing the expected affinities.

\subsection{Exact-model calibration}
\label{app:eval_synth}

We first test the Gaussian affinity reduction in the setting where its
assumptions hold. We sample balanced \(K=4\) Gaussian class-conditionals in dimension
\(d=16\) with shared covariance, vary both class separation and bandwidth
\(\varepsilon\), and compare empirical aggregate quantities with the affinity-derived
overlap predictions of Section~\ref{sec:gaussian-specialization}.

\begin{figure}[t]
    \centering
    \includegraphics[width=\linewidth]{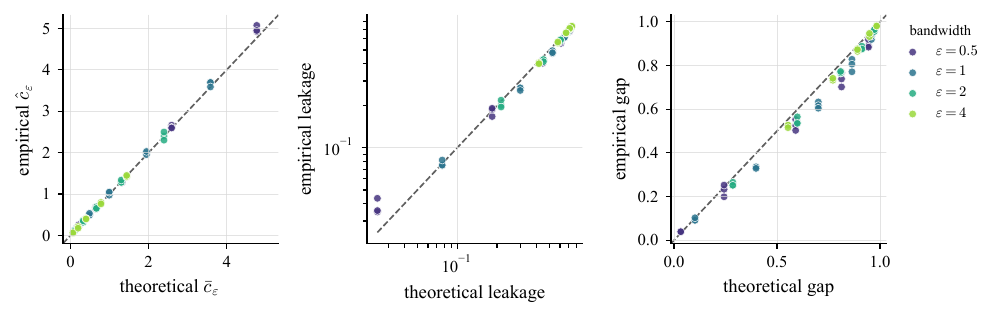}
    \caption{\textbf{Gaussian affinity calibration.}
    Affinity-derived overlap predictions versus empirical aggregate
    measurements for balanced Gaussian class-conditionals. Left: mean pairwise
    separation \(\bar c_\varepsilon\). Middle: stationary leakage. Right:
    overlap-chain gap. The dashed line is the identity.}
    \label{fig:synth_bridge}
\end{figure}

In Figure~\ref{fig:synth_bridge}, the finite-sample aggregate tracks the
affinity-derived prediction across the tested separation and bandwidth range.
The experiment therefore checks how closely the empirical aggregate recovers
the population overlap quantities derived above.

\section{Implementation details}
\label{app:implementation}

\subsection{CIFAR residual-network evaluation}
The CIFAR experiments use a torchvision ResNet-18 adapted to
$32\times32$ images. Its first convolution is replaced by a $3\times3$
stride-one layer, and the max-pooling operation is removed. Training uses
random crops and horizontal flips, followed by SGD with Nesterov momentum
$0.9$, weight decay $5\times10^{-4}$, an initial learning rate of $0.1$, and
cosine decay. CIFAR-10 is trained for 120 epochs and CIFAR-100 for 200.
Features are globally average pooled after the stem and after each residual
block group. The balanced subsets and remaining evaluation choices are given
in Section~\ref{sec:cifar-protocol}.

\subsection{Controlled Gaussian evaluation}

\paragraph{Operator construction.}
Given a feature snapshot
\[
Z=\{z_i\}_{i=1}^n\subset\mathbb R^d,
\]
we form the Gaussian affinity matrix
\[
W_{ij}
=
\exp\!\left(
-\frac{\|z_i-z_j\|^2}{4\varepsilon}
\right),
\]
the degree vector
\[
D_i=\sum_{j=1}^n W_{ij},
\]
and the row-stochastic diffusion operator
\[
P_{ij}
=
\frac{W_{ij}}{D_i}.
\]
When needed for spectral calculations, we also form the symmetric operator
\[
H
=
D^{-1/2}WD^{-1/2}.
\]
We use the raw-kernel construction where \(\alpha=0\) in the diffusion-maps normalization family.

The default bandwidth is the median heuristic
\[
\varepsilon_{\mathrm{med}}
=
\frac14
\operatorname{median}_{i<j}\|z_i-z_j\|^2.
\]
For snapshots containing more than \(2000\) points, the median is estimated
from a random subset of at most \(2000\) points.

\paragraph{Empirical separation matrix.}
Every operator experiment reported here has \(K\) classes. Let
\(\hat\mu_a\) be the empirical mean of class \(a\), and let
\(\hat\Sigma_{\mathrm{pool}}\) be the pooled within-class covariance. The
pairwise empirical separation is
\[
\hat c_\varepsilon^{(a,b)}
=
\frac14
(\hat\mu_a-\hat\mu_b)^\top
\big(
\hat\Sigma_{\mathrm{pool}}+(\varepsilon+\tau)I
\big)^{\dagger}
(\hat\mu_a-\hat\mu_b),
\qquad
a\neq b,
\]
where \(\tau=10^{-8}\) is a small numerical ridge and \({}^\dagger\)
is the numerically stable pseudoinverse used in the implementation.
The mean off-diagonal separation is
\[
\bar c_\varepsilon
=
\frac{1}{K(K-1)}
\sum_{a\neq b}
\hat c_\varepsilon^{(a,b)}.
\]

\paragraph{Class-level transport observables.}
Class transport is computed after deleting the diagonal entries of $P$ and
renormalizing each row. Mass is then averaged according to the source and
destination labels. Removing the diagonal eliminates the fixed self
contribution while preserving transitions to other samples from the same
class. The resulting matrix describes average one step class transport, with
exact lumpability determined by the condition in
Section~\ref{sec:operator-snapshots}.

The resulting aggregate matrix is denoted as \(\bar P^{(K)}\). Let \(\pi\) be its
stationary distribution. Stationary leakage is
\[
\mathrm{leakage}
=
\sum_{a=1}^K
\pi_a
\bigl(1-\bar P^{(K)}_{aa}\bigr),
\]
and
\[
\mathrm{gap}_{\mathrm{coarse}}
=
1-\max_{j\ge2}|\lambda_j(\bar P^{(K)})|.
\]

\paragraph{Label-boundary energy.}
For labels the code uses the one hot map
\[
g(i)=e_{y_i}\in\mathbb R^K,
\]
instead of the scalar digit value, and reports the total label-boundary energy
\[
E[\Gamma(g,g)]
=
\frac1n
\sum_{i=1}^n
\sum_{a=1}^K
\Gamma(g_a,g_a)(i).
\]
Expanding the definition of $\Gamma$ gives
\[
E[\Gamma(g,g)]
=
\frac{1}{n\varepsilon}
\sum_{i=1}^n
\sum_{j:\,y_j\neq y_i}
P_{ij}.
\]
The expression is the node-averaged Dirichlet energy of the one hot label
map under the diffusion operator.

\paragraph{Soft diffusion radius.}
For local scale we define
\[
r(i)^2
=
\sum_{j=1}^n
P_{ij}\|z_j-z_i\|^2,
\]
and report the root-mean-square soft diffusion radius
\[
\bar r_{\mathrm{soft}}
=
\left(
\frac1n
\sum_{i=1}^n r(i)^2
\right)^{1/2}.
\]

\paragraph{Synthetic bridge validation.}
The synthetic bridge experiment uses \(K=4\) classes in ambient dimension
\(d=16\). Class means are placed in a simplex configuration, and samples
are drawn from a shared-covariance Gaussian model
\[
\Sigma
=
\sigma^2 I+\mathrm{spike}\cdot vv^\top,
\]
with \(\sigma^2=1\) and spike strength \(0.5\). The separation parameter is
varied over
\[
\{0.8,\,1.4,\,2.0,\,2.8,\,3.8\},
\]
the bandwidth over
\[
\{0.5,\,1,\,2,\,4\},
\]
the number of samples per class is \(250\), and the code runs three trials
per configuration. For each trial, we compare empirical
\(\bar c_\varepsilon\), leakage, and
\(\mathrm{gap}_{\mathrm{coarse}}\) against the corresponding population
predictions.

\end{document}